\documentclass[letterpaper]{article}

\usepackage{graphicx}
\usepackage[ruled,linesnumbered,vlined]{algorithm2e}
\usepackage{pdfpages}
\usepackage{amsmath} 
\usepackage{amsfonts}
\usepackage{amsthm}
\usepackage[margin=1in]{geometry}
\usepackage[capitalize]{cleveref}
\usepackage{color}
\usepackage{nicefrac}
\usepackage{tikz}
\usepackage{verbatim}

\usepackage{tikz-qtree}

\usetikzlibrary{arrows}

\usepackage{natbib}
\newcommand{\citeasnoun}[1]{\citet{#1}}
\renewcommand{\cite}[1]{\citep{#1}}

\newcommand{\Real}{\mathbb{R}}
\newcommand{\Natural}{\mathbb{N}}
\newcommand{\one}{\mathbf{1}}

\newtheorem{theorem}{Theorem}[section]
\newtheorem{proposition}{Proposition}[section]

\newtheorem{lemma}{Lemma}[section]

\newcommand{\states}{\mathcal{S}}
\newcommand{\actions}{\mathcal{A}}

\newcommand{\pulls}{T}

\newcommand{\plotfig}[3]{%
\begin{figure}[t]
	\centering
	\hspace*{-18pt}\includegraphics[width=3.75in]{#1}
	\caption{#2}
    \label{#3}
\end{figure}
}
\usepackage{float}
\newfloat{alg}{t}{alg}

\newcommand{\opt}{^\star}
\newcommand{\E}[2][]{\mathbb{E}_{#1}\left[\,#2\,\right]}
\DeclareMathOperator{\betaD}{Beta}
\renewcommand{\P}[1]{\mathbb{P}\left[\,#1\,\right]}

\DeclareMathOperator{\regret}{Regret}
\DeclareMathOperator{\bayregret}{BayesRegret}
\newcommand{\ucb}{\text{UCB}}
\newcommand{\gittins}{\text{Gitt}}
\newcommand{\iopt}{i\opt_u}
\newcommand{\aopt}{a_{\iopt}}
\newcommand{\muopt}{\mu_{\iopt}}

\title{Value Directed Exploration in Multi-Armed Bandits \\
with Structured Priors}

\author{Bence Cserna \hspace{16pt} Marek Petrik \hspace{16pt} Reazul Hasan Russel \hspace{16pt} Wheeler Ruml \\ University of New Hampshire \\ Durham, NH 03824 USA \\ {\tt bence},  {\tt mpetrik}, {\tt rrussel}, {\tt ruml} at {\tt cs.unh.edu}}

\begin{document}

\maketitle
	
\begin{abstract}

Multi-armed bandits are a quintessential machine learning problem requiring the balancing of exploration and exploitation. While there has been progress in developing algorithms with strong theoretical guarantees, there has been less focus on practical near-optimal finite-time  performance. In this paper, we propose an algorithm  for Bayesian multi-armed bandits that utilizes value-function-driven online planning techniques.  Building on previous work on UCB and Gittins index, we introduce linearly-separable value functions that take both the expected return and the benefit of exploration into consideration to perform n-step lookahead. The algorithm enjoys a sub-linear performance guarantee and we present simulation results that confirm its strength in problems with structured priors. The simplicity and generality of our approach makes it a strong candidate for analyzing more complex multi-armed bandit problems.

\end{abstract}

\section{Introduction}

In the multi-armed bandit setup, a decision-maker repeatedly chooses from a finite set of actions. A reward is generated independently from a probability distribution associated with the action. The underlying reward distribution for each action is unknown to the decision-maker, but each action-reward pair can further inform future choices.  Strong performance in this setup critically depends on the balance between exploring less well-understood actions and exploiting actions thought to provide high reward.  In this sense, the problem captures the quintessence of the interplay between learning and decision-making.

Many approaches to the multi-armed bandit problem have achieved impressive theoretical and empirical results~\cite{Bubeck2012,Kuleshov2014}. There is, however, growing recognition that more wide-spread practical use will require algorithms that can better exploit structured prior information~\cite{Russo2014,Lattimore2016}. For example, consider a bandit problem in which the arms represent different levels of customer discounts such as $5\%$ or $10\%$. The conversion probabilities for the discounts are not known before the promotion starts, but one can safely assume that more customers choose to buy the product when offered a $10\%$ rather than a $5\%$ discount. Such prior information should ideally be used to increase the efficiency of exploration in particular when a large number of discounts is considered. While in some problems such prior information can be captured using parametric models, like GLM-UCB~\cite{Filippi2010}, such models make many additional assumptions that are difficult to verify when little or no data is available. 

In this paper, we propose a Bayesian bandit algorithm designed to use structured prior information in order to achieve good short-term performance with a small number of arm pulls. We take an approach based on online planning, using lookahead search to consider the states to which each possible sequence of actions might lead. Relying on lookahead search makes it possible to easily exploit structured prior information when it is available. Because the state space grows exponentiaith the number of arms, it is impossible to enumerate all reachable states to the problem horizon. Thus, this formulation of the problem relies crucially on having a value function that can be applied at a modest depth cut-off in lieu of further state search.

Many methods for computing approximate value functions have been developed in reinforcement learning and have, in fact, been used to solve some multi-armed bandit problems~\cite{Whittle1988,Adelman2008}. But they can be computationally intensive and cumbersome to use. As our main contribution, we propose in \cref{sec:ucb_method} a new method for computing \emph{linearly-separable value functions} that, when used in concert with lookahead search, performs as well as state-of-the-art-algorithms. The method computes value functions by exploiting existing algorithms and the weakly-coupled property of multi-armed bandit problems. It also enjoys sublinear regret as we show in \cref{sec:regret_analysis}. Our algorithm is simple to implement and, as we demonstrate in \cref{sec:experiments}, performs well in bandit problems with structured prior information.  

Given the fundamental simplicity of our approach and its empirical success, we are optimistic that it may provide a basis for addressing more complex problems, such as contextual bandits, in the future. Our approach also opens the door to bandit algorithms that can yield improved performance when additional computation time is available.

\section{Background}

We begin by describing the Bayesian multi-armed bandit problem in more detail.  We focus on the case of Bernoulli bandits, deferring discussion of more complex models to \cref{sec:use_case}.  We then briefly review previously proposed algorithms before turning to our new method.

\subsection{Problem Formulation} \label{sec:problem_formulation}

The decision-maker in the standard multi-armed bandit problem aims to maximize the cumulative return by repeatedly choosing one of $N$ arms: $\actions = \{a_1, \ldots, a_N \}$. Choosing an arm $a_i$ results in receiving a reward $R_i \in \{0,1\}$ distributed according to a Bernoulli distribution with a mean $\mu_a$. The mean $\mu_a$ is not known in advance. To achieve the maximal cumulative return over a horizon of $T$ steps, the decision-maker must balance exploration to learn about the expected returns of arms with exploitation in order to learn which arms are more likely to provide high rewards.
 
In the Bayesian variant of the problem, the decision-maker has access to a prior distribution over the expected reward $\mu_1, \mu_2, \ldots, \mu_N$ for each arm $a_1, a_2, \ldots, a_N$.
We use $\mu = (\mu_1, \ldots, \mu_N)$ to represent the prior parameters of the bandit; each $\mu_i$ is distributed according to a Beta distribution. As in most machine learning settings, the Bayesian approach has both advantages and disadvantage ---a proper discussion is beyond the scope of this paper and we refer to \citeasnoun{Kaufmann2012a,Russo2014b,Kim2015} and the references therein for details.

\begin{figure*}[t]
	\centering
	\begin{tikzpicture}[
	sibling distance=20pt,
	level 1/.append style={level distance=40pt},
	level 2/.append style={level distance=60pt},
	edge from parent/.append style={->,draw,>=stealth'}, align=center,
	state/.style = {shape=rectangle,draw, align=center,fill=blue!20},
	action/.style = {shape=rectangle,draw, rounded corners=2mm, align=center,fill=green!20}]

	\Tree[ .\node[state] {$(\alpha_1 , \beta_1) ~\Vert~ (\alpha_2 , \beta_2) $}; 
	 	[
	 	.\node[action] {Pull arm $a_1$}; 
	 		\edge node {$R=1$\\$P=\frac{\alpha_1}{\alpha_1+\beta_1}$};
	 		[.\node[state] {$(\alpha_1+1, \alpha_2) ~\Vert~ (\beta_1 , \beta_2)$};  \node {\ldots};]
	 		\edge node {$R=0$\\$P=\frac{\beta_1}{\alpha_1+\beta_1}$};
	 		[.\node[state] {$(\alpha_1 , \beta_1 + 1) ~\Vert~ (\alpha_2 , \beta_2)$};  \node {\ldots};]
	 	]
	 	[
	 	.\node[action] {Pull arm $a_2$};
	 		\edge node {$R=1$\\$P=\frac{\alpha_2}{\alpha_2+\beta_2}$};
	 		[.\node[state] {$(\alpha_1 , \beta_1) ~\Vert~ (\alpha_2 + 1 , \beta_2)$};  \node {\ldots};]
	 		\edge node {$R=0$\\$P=\frac{\beta_2}{\alpha_2+\beta_2}$};	 		
	 		[.\node[state] {$(\alpha_1 , \beta_1) ~\Vert~ (\alpha_2 , \beta_2+1)$};  \node {\ldots};]
	  	]
	]	
	\end{tikzpicture}
	\caption{A single transition of the Bernoulli multi-armed bandit problem.} \label{fig:state_illustration}
\end{figure*}
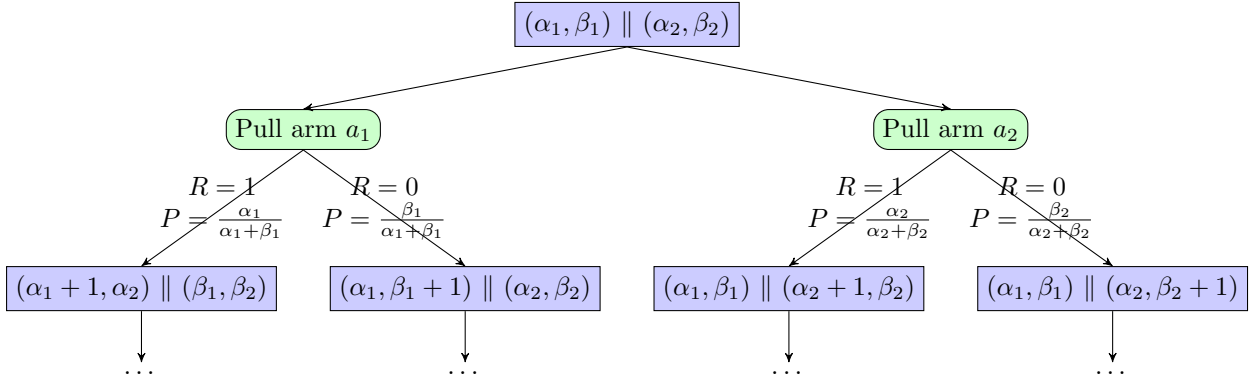

The Bayesian multi-armed bandit problem can be modeled as a Markov Decision Process~(MDP).  The state in this MDP represents the sufficient statistic of the history of the observed rewards for each arm. We denote the state space of the MDP that represents the multi-armed bandit problem as $\states = \states_1 \, \cup \,\states_2\, \cup \ldots \cup \,\states_T$, where $\states_t$ is the set of states at time $t$. The actions in this MDP are simply the pulls of arms of the bandit.  \cref{fig:state_illustration} shows a fragment of the MDP for a two-arm bandit, illustrating the transition from one state to the four possible successors.

In order for the Markov property to hold, each state of the MDP must represent the posterior distribution of $\mu_i$ given the history of rewards for every arm $a_i$. The posterior parameter $\mu_i$ is distributed according to the Beta distribution $\betaD(\alpha,\beta)$ with some parameters $\alpha,\beta$ because it is the conjugate prior to Bernoulli distribution~(e.g., \cite{Gittins2011}). Any state $s\in\states$  can, therefore, be represented as \[ s = \Bigl((\alpha_1, \beta_1), (\alpha_2,\beta_2), \ldots, (\alpha_N,\beta_N) \Bigr) ~,\]
where $\alpha_i,\beta_i$ represent the Beta distribution parameters for each arm $i$. We use $\alpha_i^0, \beta_i^0$ to denote the parameters of the prior Beta distributions and generally assume that $\alpha_i^0 = \beta_i^0 = 1$ which corresponds to the uniform prior.  The parameters $\alpha,\beta$ of the Beta distribution have a convenient interpretation: after observing $n_s$ successes (value 1) and $n_f$ failures (value 0) for $R_i$ then $\mu_i \sim \betaD(\alpha_i^0 +n_s, \beta_i^0 + n_f)$. Thus the transition, after pulling an arm $a_i$, consists of merely adding one to the appropriate $\alpha_i$ or $\beta_i$ based on the observed reward.

When the bandit is in state $s_t$ and the decision maker chooses arm $a_i$, the subsequent state is represented by a random variable $S_{t+1}(s_t,a_i)$. When  $s_t = (\ldots, (\alpha_i,\beta_i), \ldots)$, this random variable is distributed as
\begin{align}
\label{eq:success_prob}
\P{S_{t+1} = (\ldots, (\alpha_i+1,\beta_i), \ldots)} &= \frac{\alpha_i}{\alpha_i+\beta_i}~, \\
\label{eq:fail_prob}
\P{S_{t+1} = (\ldots, (\alpha_i,\beta_i+1), \ldots)} &= \frac{\beta_i}{\alpha_i+\beta_i} ~,
\end{align}
where the transition probabilities follow from the definition of the mean of the Beta distribution. To reduce clutter, we omit $s_t$ and $a_i$ when they are obvious from the context. The rewards received in transitions  \eqref{eq:success_prob} and \eqref{eq:fail_prob} are 1 and 0 respectively.

By specifying which arm to pull at any given state, a multi-armed bandit algorithm defines a policy for this bandit MDP. The decision rule at time $t$ is denoted as $\pi_t: \states_t \to \actions$ and the policy is the collection of the decision rules $\pi = \{ \pi_t \mid t=1\ldots T\}$ for each time step $t$.  The celebrated Gittins index defines the optimal policy for the discounted infinite-horizon version of the bandit problem, but that method is not based on directly solving the MDP. In most practical problems, it is impossible to compute the optimal policy because the number of states grows exponentially with the number of arms.  Unfortunately, while Gittins index may be optimal, it provably does not generalize to most other bandit problems.

The established performance measure for classic bandit algorithms is the regret, sometimes referred to as pseudo-regret~\cite{Bubeck2012}, which is defined for a particular realization of the bandit parameters $\mu$ and policy $\pi$ as
\[ \regret(\pi,T,\mu) = \sum_{t=1}^T \Bigl( \max_{i=1\ldots N} \mu_i - \E{R_{\pi(S_t)} \mid \mu} \Bigr)~, \] 
where $S_t$ is a random variable that represents the state of the bandit process. In Bayesian bandits, it is natural to instead evaluate Bayesian regret:
\[  
\bayregret(\pi, T) = \E[\mu]{\regret(T,\mu)}~.
\]
We aim to minimize Bayesian regret with a particular focus on the regret in the first few steps. While the guarantees provided by a small bound on the Bayesian regret are somewhat weaker than that of regular regret, it is a very reasonable measure in most circumstances.

\subsection{Previous Work}

The literature on bandit problems is enormous, so we will focus on just the most relevant algorithms. The UCB family of algorithms \cite{Auer2002} use the problem structure to derive tight optimistic upper bounds.  While these algorithms are simple and have been used in various applications with success, they lack the ability to incorporate structured prior information such as arm dependency or different reward policies without requiring complex and difficult re-analysis of the upper bounds. \citeasnoun{Kaufmann2012a} propose Bayes-UCB, a Bayesian index policy that improves on UCB in Bayesian bandits by taking advantage of the prior distribution. 

\citeasnoun{Russo2014} describes an approach to addressing the limitations of the optimistic approach that serves as the basis for the UCB family of algorithms. They describe and method that considers not only the immediate single-period regret but also the information gain to learn from the partial feedback and to optimize the exploration-exploitation trade online. They provide a strong Bayesian regret bound that applies for a general class of models. Our new method is based on a similar principle but uses and additive value functions to estimate the information gain.

Thompson sampling works by choosing an arm based on its probability of being the best arm.  Concretely, the method draws a sample from the decision maker's current belief distribution for each arm and then chooses the arm that yielded the highest sample. The performance of Thompson sampling has been proved to be near optimal, and it is simple and efficient to implement. Thompson sampling can easily be adapted to a wide range of problem structures and prior distributions~\cite{Agrawal2011,Leike2016,Kaufmann2012}.  For example, one can reject sets of samples that contradict contextual information.  However, the simplicity of the method makes it also difficult to improve its performance.

Gittins indices exploit the weak dependence between actions to compute the optimal action in time that is linear in the number of arms~\cite{Gittins1979,Chakravorty2014}. Gittins indices, however, are guaranteed to be optimal only for the basic multi-armed bandit problem, require a discounted infinite-horizon objective, and provably cannot be extended to most interesting and practical problems which involve correlations between arms or an additional context~\cite{Gittins2011}.

\section{Constructing Value Functions} \label{sec:ucb_method}

We now turn to our new approach, which we call ``ELSV'' (Exploration via Linearly Separable Value Functions). As described above, our main goal is a method that is flexible and takes advantage of the complex problem structure or prior knowledge in order to reduce the regret. We achieve this by taking a state-space search-based approach and by leveraging the exploration-exploitation trade-off behavior of existing algorithms to build good value functions. 

The ELSV algorithm, as described in this section, does not improve the performance of existing methods when applied to Bernoulli bandits. In \cref{sec:experiments}, we describe how it can be extended easily to settings with prior information in which it significantly outperforms state-of-the-art methods.

\begin{algorithm}  
	\KwIn{Current time step $t$, current state $s_t$, and value function $v_t: \states_t \rightarrow \Real$}
	\KwOut{Arm to pull at time step $t$} 
	\For{$a\in\actions$}{
		$q_t(s_t,a) \gets \E{r(s_t,a,S_{t+1}) + v_{t+1} (S_{t+1})}$ \;
	}
	\Return{$\arg\max_{a\in\actions} q_t(s_t,a)$ }\;
	\caption{One-step lookahead algorithm}    \label{alg:lookahead}
\end{algorithm}

Our general approach is based on an $n$-step lookahead guided by a specific value function. This is an instance of receding horizon control, a common approach to solving online planning and reinforcement learning problems~\cite{Sutton2016}.  The state space is enumerated to depth $n$, at which point a value function is evaluated at the frontier states to avoid further expansion.  Note that, because there are multiple ways of reaching a state in the MDP, the state space forms a graph and it is important to detect and merge duplicate states.  The values are backed up to the current state at the root and the best-looking action is chosen.  After the outcome of pulling the arm is observed, the cycle repeats again with a fresh lookahead.  A simplified version of the algorithm, depicted in \cref{alg:lookahead}, estimates the value of each action by computing the expected value for the next step.

Since \cref{alg:lookahead} does not rely on any complex confidence bounds, one would expect that it can easily generalize to many different problems. Choosing a longer lookahead horizon also offers the promise of trading off computational time for more efficient exploration. The quality of this algorithm will clearly depend on whether it is supplied with a good value function $v$.
 
There has been little previous work that considered a value function-driven approach to bandit problems, with \citeasnoun{Adelman2008} being one notable exception. This is perhaps because UCB is simple and efficient, while computing value functions via approximate dynamic programming requires complex computation and can be unreliable. We show, however, that it is possible to \emph{efficiently} construct good value functions directly from UCB and other popular bandit algorithms. Surprisingly, such value functions are simple and linearly separate over the individual arms. 

Before describing how we construct the value function, consider what it is supposed to represent. Consider, for example, a state $s=((2,3),(10,10))$ in a two-arm bandit problem with $10$ steps remaining until the end of the horizon $T$ is reached. The expected returns of the two arms are $\nicefrac{2}{5}\cdot 10 = 4$ and $\nicefrac{1}{2} \cdot 10 = 5$ respectively. One could simply assign $v(10) = \max\{ 4, 5\} = 5$, but this would not be precise. The first arm, while apparently having a lower expected mean, is far less certain than the second arm (because of a smaller number of pulls). In order to achieve good results, and in particular a sub-linear regret guarantee, the value function must consider not only the expected return but also the confidence of the estimates. Another way to put it is that the value function must model both the expected return (exploitation) and the benefit of exploration.

\begin{algorithm}
	\KwIn{Current time step $t$, current state $s_t$, and index function $z_t: \states_t \times \actions \rightarrow \Real$}
	\KwOut{Arm to pull at time step $t$} 	
	\Return{$\arg\max_{a\in\actions} z_t(s_t^i,a_i)$ }\;
   	\caption{Index Policy} \label{alg:index}
\end{algorithm} 

We seek to take advantage of the implicit value of exploration that is encoded by existing multi-armed bandit index algorithm. \cref{alg:index} shows a canonical example of an index-based algorithm. UCB, Gittins index, and many other methods fit this basic mold. Note that the index $z_t(s_t^i, a_i)$ is computed for each arm separately. The notation  $s_t^i$ denotes the component of state $s_t$ that corresponds to arm $a_i$, that is $s_t^i = (\alpha_i,\beta_i)$. For example, if $s_t = ((5,2),(4,6))$ then $s_t^2 = (4,6)$. An important property of the index function $z_t(s_t^i,a_i)$ is that it is completely independent of the states of other arms and can thus be computed efficiently. 

The challenging part of constructing an index algorithm is, obviously, how to define the index value $z_t$.  
If we use the Bayesian expectation of the immediate return in place of the standard frequentist one, then the value of the index for the $\alpha$-UCB algorithm for each arm $a\in\actions$ is
\begin{small}
\begin{equation} \label{eq:ucb_index}
\begin{aligned}
z_t^{\ucb}(s_t^i,a_i) &= r(s_t^i, a_i) + \sqrt{ \frac{\alpha \log t}{\pulls_i(s_t^i)} } \\
&= \underbrace{r(s_t^i, a_i)}_{\text{Immediate reward}} + \underbrace{b_t^{\ucb}(s_t^i,a_i)}_{\text{Exploration bonus}}
\end{aligned}
\end{equation}
\end{small}
where $r(s_t^i,a_i) = \E{r(s_t^i,a_i,S_{t+1}^i)} = \nicefrac{\alpha_i}{\alpha_i+\beta_i}$ is the expected reward after pulling arm $a_i$, $\pulls_i(s_t^i) = \alpha_i + \beta_i - 2$ is the number of times the arms has been pulled (recall that the initial states is $\alpha_i^0 = \beta_i^0 = 1$), and $b_t^{\ucb}$ is the exploration bonus. 

The classic UCB algorithm uses $\alpha=2$, but sub-linear regret can be in fact shown with $\alpha > 1$~\cite{Bubeck2012}. Lower values of $\alpha$ typically lead to better empirical performance but also make it more difficult to bound the regret. We use $\alpha=1$ unless otherwise specified. 

Another celebrated example of an index policy is the Gittins index, see for example~\cite{Gittins2011}. While UCB is asymptotically optimal (up to a constant factor), following the Gittins index results in an optimal policy in several simple bandit settings. For example, Gittins indices are optimal for an infinite-horizon discounted Bernoulli bandit problem. We generally use discount $\gamma = 0.99$ and the horizon of $1000$ when approximate the infinite horizon.

Unlike UCB, Gittins index does not have a closed-form expression, but it instead needs to be precomputed. Since the index is computed for each arm independently, it can be computed and used efficiently regardless of the number of arms in the bandit problem~\cite{Nino-Mora2011}. We use $z^\gittins_t(s_t^i,a)$ to denote the value of the Gittins index and $b^\gittins_t(s^i_t,a_i) = z^\gittins_t(s^i_t,a_i) - r(s_t^i, a_i)$ to denote the exploration benefit that it assigns to the arms.  Many other multi-armed bandit methods have been proposed, some examples are Bayes-UCB~\cite{Kaufmann2012a} or UCB-V~\cite{Audibert2009} to name a few.

We are now ready to describe ELSV, the new method for constructing linearly-separable value functions. A linearly-separable value function for components $\upsilon_t^i: \states_t^i \rightarrow \Real$ is such that, for each $t\in\mathcal{T}$ and for each state $s_t\in \states_t$,
\begin{equation} \label{eq:separable}
v_t(s_t) = \sum_{a\in\actions} \upsilon_t^i(s_t^i)~.
\end{equation}
Linearly separable value functions are attractive due to their simplicity and have been used widely in reinforcement learning and approximate dynamic programming~\cite{Powell2008,Powell2004,Rust1994} and in previously for approximating the value function in bandit problems~\cite{Adelman2008}.

We begin with an arbitrary bandit index function $z_t^i$ and the corresponding exploration bonus function $b_t^i$. Each component $\upsilon_t^i$ must satisfy the following condition for every $a_i$, $t$, and $\tau\le t$:
\begin{small}
\begin{equation} \label{eq:dynamic_update}
\begin{aligned}
\upsilon_{t+1}^i(s_\tau^i) &= \E{\upsilon_{t+1}^i(S_{\tau+1}^i)} + r(s_\tau^i,a_i) - z_t(s_\tau^i,a_i) \\
&=\E{\upsilon_{t+1}^i(S_{\tau+1}^i)} - b_t(s_\tau^i,a_i) ~,
\end{aligned}
\end{equation}
\end{small}
where $S_{\tau+1}^i$ is short for $S_{\tau+1}^i(s_\tau,a_i)$, the random variable that represents the state following the pull of arm $a$. It is important to note that all $\upsilon$'s involved use the same time index $t+1$ while the states are over two time steps $\tau$ and $\tau+1$.

To understand the requirement in \eqref{eq:dynamic_update} more intuitively, we can rewrite it as
\[ \E{\upsilon_{t+1}^i(S_{\tau+1}^i)} - \upsilon_{t+1}^i(s_\tau^i) = b_t(s_\tau^i,a_i) ~. \]
The term $\E{\upsilon_{t+1}^i(S_{\tau+1}^i)}$ represents the expected value of the state at time $t+1$ after pulling the arm $a_i$. In contrast, the term $\upsilon_{t+1}^i(s_\tau^i)$ represents the expected value of the state at time $t+1$ when arm $a_i$ is not pulled. The difference between these two terms is the change in the value of the current state, or in other words, how much we have learned about the arm after pulling it. And this increase in information about the arm should be equal to the exploration bonus assigned by the index.

\begin{algorithm}
	\KwIn{Arm $a_i$, time step $t\in 1,\ldots,T$, and exploration bonus function $b_t: \states_t \times \actions \rightarrow \Real$}
	\KwOut{Value functions $\upsilon^i_1, \ldots, \upsilon^i_T$ for arm $a_i$}	

   	$\states_t \leftarrow \Bigl\{ (\alpha,\beta) \in \Natural_{>0}^2 ~\vert~ \alpha + \beta - 2 \le t-1 \Bigr\}$ \;	
   	$\upsilon_t^i(s) \leftarrow 0$ \quad $\forall s\in \states_t$\;
	\For{$\tau = t-2$ \KwTo $1$}{
       	\ForEach{$(\alpha,\beta) \in \{s\in\states_\tau ~\vert~ \alpha+\beta-2 = \tau \}$}{
       		$p \leftarrow \frac{\alpha}{\alpha+\beta}, \quad q \leftarrow \frac{\beta}{\alpha+\beta}$ \;
       		$\upsilon_t^i(\alpha,\beta) \leftarrow  p\cdot \upsilon_t^i(\alpha+1,\beta) + q \cdot \upsilon_t^i(\alpha,\beta+1) - b_t(\alpha, \beta)$ \;
        } 	
	}
	\Return{$\upsilon_t^i$ }
	\caption{ELSV: Computing linearly separable value functions that satisfy \eqref{eq:dynamic_update}.} \label{alg:value_computation}
\end{algorithm} 

\begin{theorem} \label{thm:value_property}
Let $\pi_I$ represent \cref{alg:index} with index function $\hat{z}_t$. Suppose a policy $\pi_V$ represents \cref{alg:lookahead} with value function $\hat{v}_t$ as defined in \eqref{eq:separable}--\eqref{eq:dynamic_update} using an index function $\hat{z}_t^i$. Then $\pi_V(s_t) = \pi_I(s_t)$ for each $s_t\in\states_t$.
\end{theorem}
We defer the proof of \cref{thm:value_property} to \cref{sec:proof_value_property}. It follows by comparing the value of pulling an arm with the value of a hypothetical state which would have resulted from not pulling any arm. The argument relies on the fact that a policy is not affected by adding or subtracting a constant from the value function for all states in $\states_t$ for any $t$.

The simplest arm index would simply ignore the exploration bonus by setting it to $b_t(s_t^i,a_i) = 0$. The value function for this setting will be a constant and the 1-step lookahead policy will simply be greedy. Having no exploration bonus, therefore, leads to no exploration and pure exploitation.

The value function in \cref{thm:value_property} induces the same policy as the index, but it is still an approximation of the true value of each state. Therefore, a value function $\upsilon$ constructed from the Gittins index will lead to the optimal policy (for the discounted infinite-horizon bandit) but it is not the optimal value function. 

\cref{alg:value_computation} describes a dynamic programming method that can be used to compute a value function that satisfies \eqref{eq:dynamic_update}. The following proposition, which can be shown readily by algebraic manipulation, states the complexity of the algorithm.
\begin{proposition} \label{prop:computable}
The linearly separable value function $\upsilon_t^i$ can be computed using dynamic programming with computational complexity $O(t^3)$. 
\end{proposition}
Since the values $\upsilon$ are computed independently for each $t$ they do not have to be pre-computed ahead of time but can be computed on as needed basis and only for the states relevant to choosing an action. It is also important to note that the complexity in \cref{prop:computable} is independent of the number of arms and the complexity of the 1-step lookahead in \cref{alg:lookahead} is linear in the number of arms. ELSV can scale to a large number of arms with no significant difficulties.

\begin{figure*}
	\centering	
	\includegraphics[width=0.45\linewidth]{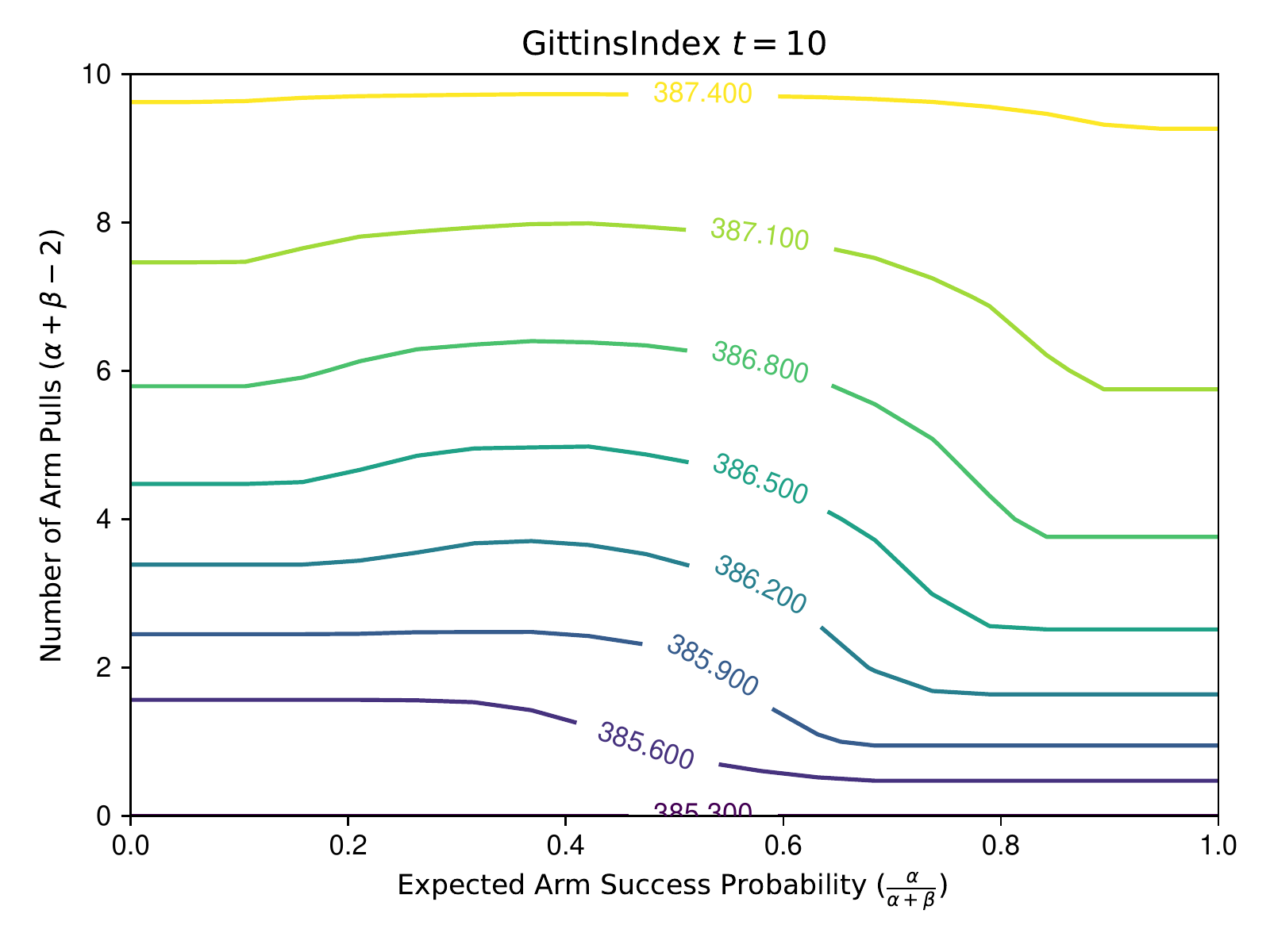}
	\hfill
	\includegraphics[width=0.45\linewidth]{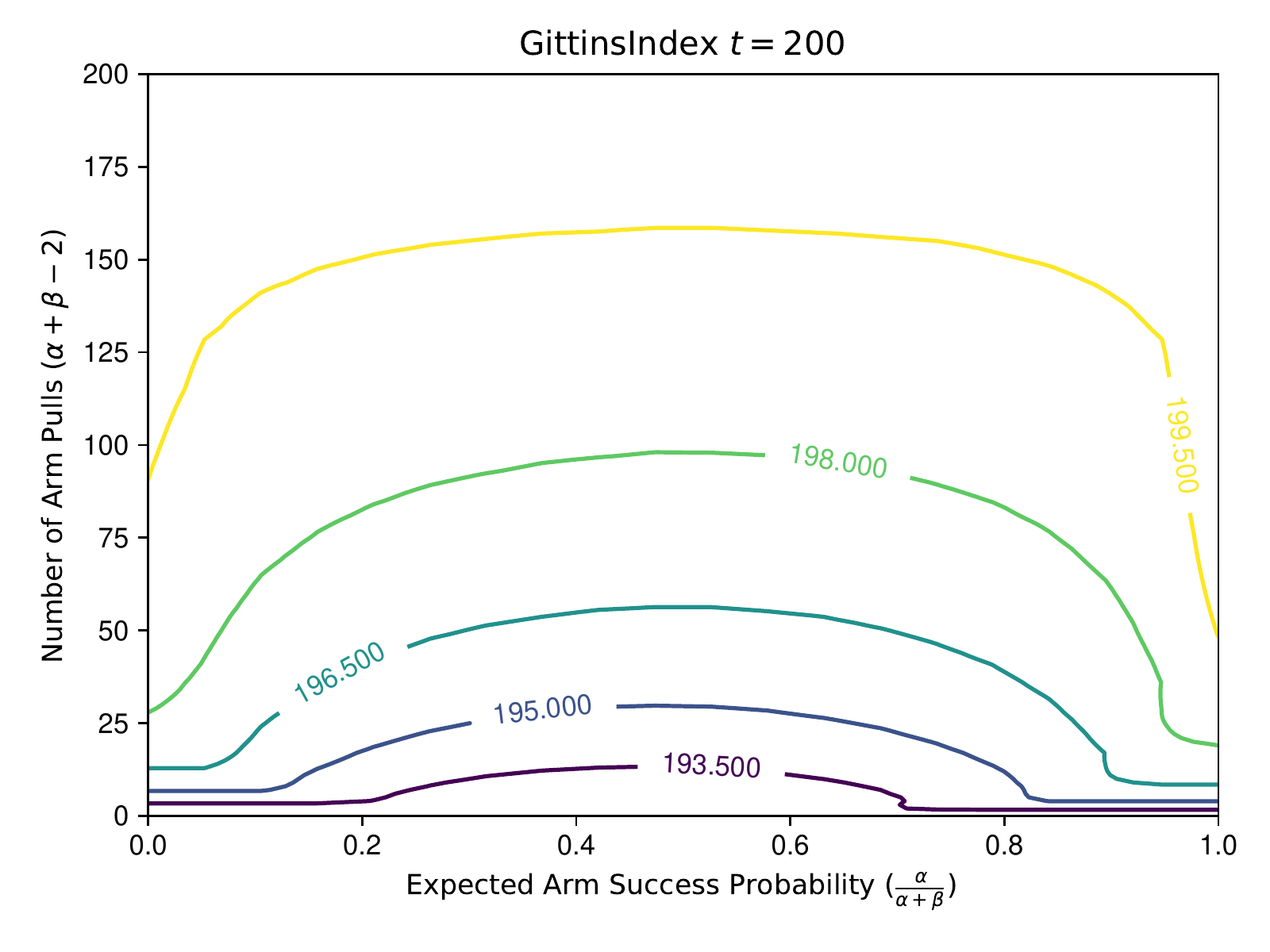}
	\caption{Value functions for Gittins index at $t=10$ and $t=200$ (after $10$ and $200$ pulls of some arm). } \label{fig:value_functions_git}
\end{figure*}

\begin{figure}
	\centering
	\includegraphics[width=0.7\linewidth]{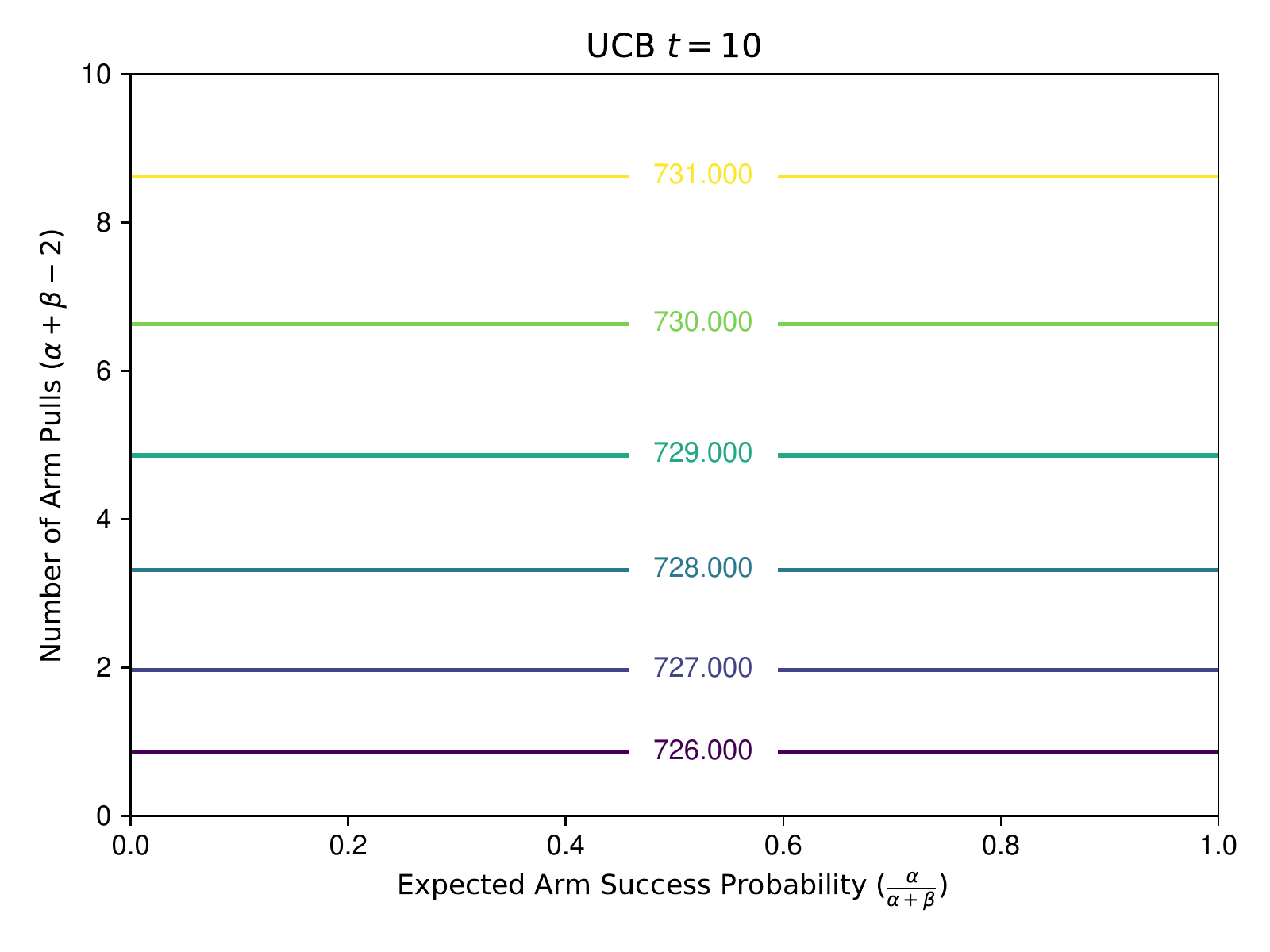}
	\caption{Value function for UCB index at $t=10$  (after $10$ pulls of some arm). } \label{fig:value_functions_ucb}
\end{figure}

\cref{fig:value_functions_git} depicts value functions computed by ELSV for Gittins index at $t=10$ and $t=200$ for each state. The states in $\states$ are mapped to a 2-dimensional space and the contours indicate the value function at that state. The number of arm pulls on the vertical axis can be smaller than $t$ since the arm may not be pulled in every time step. \cref{fig:value_functions_ucb} shows the value function computed by ELSV for UCB index for comparison. As noted above, the constant offset of the value functions is irrelevant to the quality of the policy. The value functions in the plots are offset to satisfy,
\[ \upsilon^i_t(s_t) \ge \max \left\{r(s_t^i,a_i) + \E{\upsilon_{t+1}^i(S_{t+1}^i)}, \upsilon_{t+1}^i(s_t^i) \right\}~, \]
which leads to more reasonable value estimates without influencing the policy.

The value function of UCB is notably simpler than the one for Gittins index. As expected, the UCB value function is \emph{concave} and increasing but independent of the expected success probability. This indicates that exploration in UCB is really driven just by the immediate reward and the certainty in it---the potential long-term benefits of the arm are ignored. On the other hand, the value function for the Gittins index value exhibits a curious structure: it increases toward both low and high probabilities. This is counterintuitive as one would expect the value function to monotonically increase with the success probability. In a multi-armed bandit, however, it is also valuable to learn that an arm is not good which reduces the need for further exploration. Arms with medium success probabilities do not provide high rewards and yet require significant exploration. Notice also that this property is much more exaggerated at $t=200$.

In the remainder of the paper, we discuss regret bounds ELSV and evaluate its performance experimentally on a bandit problem with additional problem structure.

\section{Analysis of Regret} \label{sec:regret_analysis}

In this section, we prove that ELSV with the UCB value function has sublinear regret. The sublinear bound on the regret is not surprising; \cref{thm:regret} shows that ELSV with such a value function behaves identically to UCB which enjoys sublinear regret bounds. Instead, the main goal of this section is to establish a new methodology that can be used to analyze regret of multi-armed bandit algorithms driven by value functions.

Our goal is, in particular, to derive regret bounds that depend on some property of the value function used by ELSV. We need additional notation to describe such a property concisely. Let $q_t(s,a)$ stand for the expected value after pulling an action $a$ in state $s_t$:
\[ q_t(s_t,a) = r(s_t,a) + \E{v_{t+1}\Bigl(S_{t+1}(s_t,a)\Bigr)} ~.\]
As we show below, to bound regret it is sufficient to establish an upper bound on:
\begin{equation} \label{eq:residual_term}
\varphi_t(\mu,s_t,a_i) = q_t(s_t,a_i) - (\mu_i + v_{t+1}(s_t))~.
\end{equation} 
This value $\varphi_t$ compares the estimate of the expected value $q_t(s_t,a_i)$ with a more precise estimate of the same value $(\mu_i + v_{t+1}(s_t))$. The more precise estimate uses the unknown parameter $\mu$. One could also interpret $\varphi_t$ as a finite-horizon form of the Bellman residual used in bounds on performance loss in reinforcement learning~(e.g. \cite{Petrik2011b}).  

The following lemma plays a key role in establishing the regret bounds. It shows how the regret of ELSV can be decomposed into two components: one is independent of the algorithm and the other one is independent of the optimal action.
\begin{lemma} \label{lem:decomposition}
The regret for any policy $\pi_t$ computed using 1-step lookahead with respect to value function $v_t$ is upper bounded as follows:
\begin{gather*}
\regret(\pi, T, \mu) \le \sum_{t=1}^T \E[S_t]{\varphi_t(\mu,S_t,\pi(S_t))} - \\
- \sum_{t=1}^T \E[S_t]{\varphi_t(\mu,S_t,a_{i_\mu\opt})} ~,
\end{gather*}
where  $S_t$ is the random variable representing the state at time $t$ under policy $\pi$ and $i\opt_\mu = \arg\max_{i \in \actions} \mu_i$ is the optimal action for the unknown parameter $\mu$.
\end{lemma}
The proof of the lemma, deferred to \cref{sec:regret_decomposition_proof}, follows readily from the fact that ELSV chooses actions that maximize $q_t$. 

The actual regret bound depends, of course, on the value function that is used. We now use \cref{lem:decomposition} to bound the regret of the policy that uses UCB derived value function $v^\ucb$ as described in \eqref{eq:dynamic_update}. 

\begin{theorem} \label{thm:regret}
The regret of policy $\pi_U$ of \cref{alg:lookahead} that uses $\hat{v}^\ucb$ (with $\alpha = 2$) is bounded as:
\[ \regret(\pi_U, T, \mu ) \le O(\sqrt{T \log(T) })~.  \]
\end{theorem}
The proof of the theorem can be found in \cref{sec:thm_regret_proof}. Note that this bound is not tighter than existing bounds on UCB algorithms, but it does establish sublinear regret of ELSV.

To establish the bound in \cref{thm:regret}, it is necessary that the difference of the residuals \eqref{eq:residual_term} for the arm chosen by ELSV and the optimal arm are not only small but must also decrease at least quickly $1/\sqrt{t}$. In other words, it is not sufficient for the errors in \eqref{eq:residual_term}
to be small, they also must decrease as more information about the returns of the arms becomes available.

\section{Experimental Results} \label{sec:experiments}

In this section, we compare the performance of ELSV to that of UCB, Bayes-UCB, Thompson sampling, and Gittins indices in simulation. We first analyze in \cref{sec:bernoulli_bandits} the impact of the lookahead horizon on the performance in the plain Bernoulli bandit setting. Then, in \cref{sec:use_case}, we describe the application to a problem in which structured prior information is available.

Unlike UCB and Thompson sampling, Gittins index must be pre-computed in advance. It is also optimal only for infinite-horizon discounted problems. The index that we use in our experiments was computed with a discount factor $\gamma=0.99$ and horizon of $1000$ to approximate the infinite-horizon value. We computed the index using the calibration method with the step size of $0.001$ as described, for example, in \citeasnoun{Nino-Mora2011}. Other approximations for computing the index have been proposed recently \cite{Gutin2016,Lattimore2016a}.

\subsection{Bernoulli Bandits} \label{sec:bernoulli_bandits}

Our first set of experiments is in the standard Bernoulli bandit setting is described above in \cref{sec:problem_formulation}. As shown in \cref{thm:value_property}, ELSV's performance will be identical to the index algorithm it is based upon. And our experiments indeed confirm this fact.



\plotfig{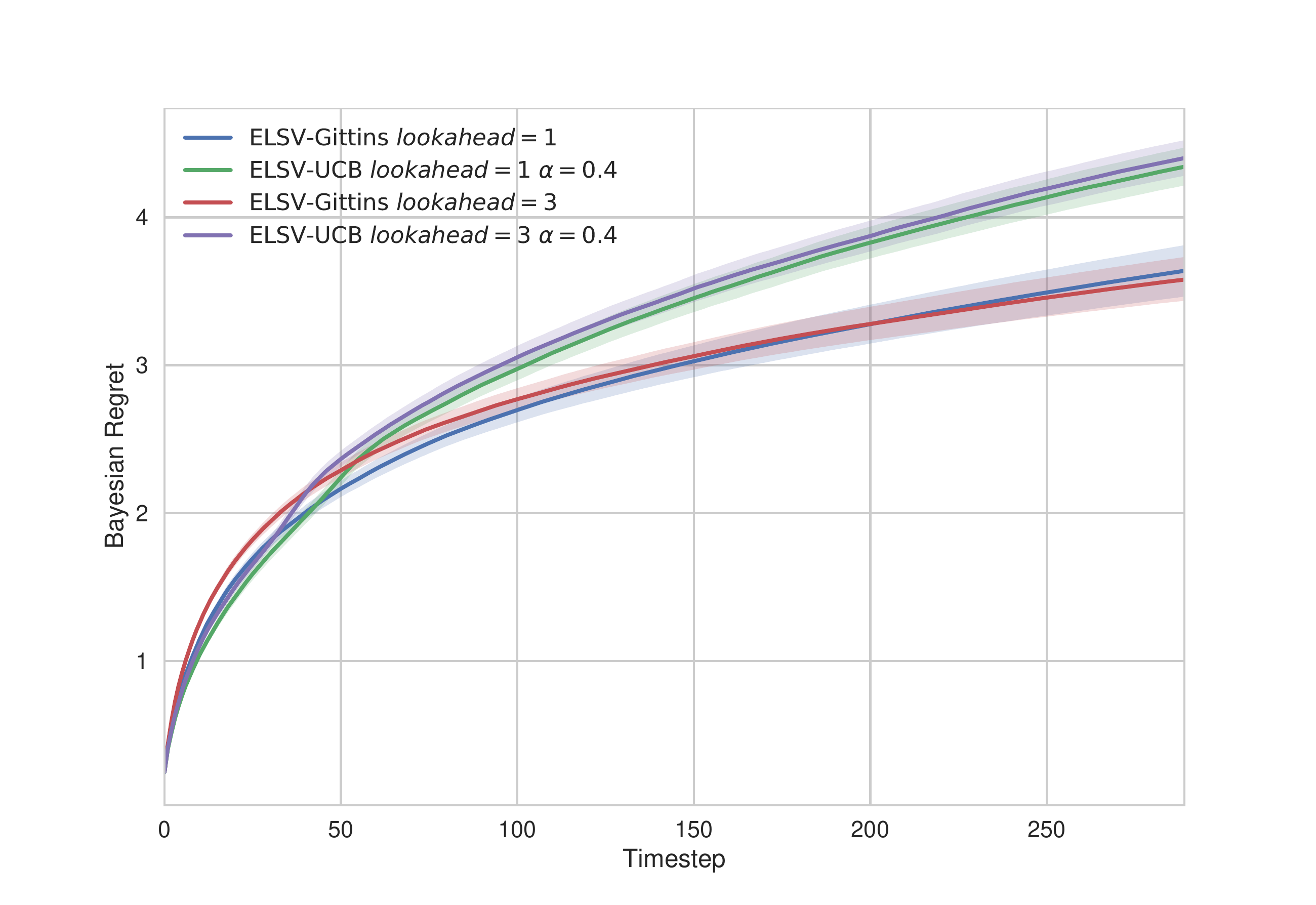}{Bernoulli bandit regret with lookahead of 1 and 3 steps.}{fig:lookahead}

The main purpose of the experimental evaluation in this section is to understand the effect of the size of the lookahead on the performance of the algorithm. A reasonable assumption in online planning algorithms is that their performance generally improves with an increasing horizon size. \cref{fig:lookahead} compares the Bayesian regret of 1-step and 3-step lookaheads in a 3-armed problem with value functions computed by ELSV. We use two value functions, one computed from the Gittins index and another one from $\alpha$-UCB with $\alpha=0.4$ (this $\alpha$ is unrelated to the $\alpha$ value used in each state). The results are averaged over 5200 problem instances with arm success probabilities drawn from the uniform $\betaD$ distribution and the shaded areas around the curves show $95\%$ confidence intervals. 

\cref{fig:lookahead} highlights a surprising finding: longer lookahead does not reduce the regret. We observed virtually no improvement in the regret up to a horizon 10 at which point the search becomes computationally intractable. We hypothesize that a more careful, focused, and deeper search would be more likely to yield improvements.

\subsection{Constrained Bernoulli Bandits} \label{sec:use_case}

The constrained Bernoulli bandit problem represents a more challenging case that is not handled well by existing algorithms.  As described in the introduction, this problem is motivated by an application when trying to optimizing the level of \emph{personalized} discount offers to customers in an e-commerce setting. It has been studied extensively in operations research using customer choice models~\cite{Train2003}. Although such choice models can be combined with UCB methods, their application with no historical data is often problematic. For the purpose of these experiments, we simply assume that the success probabilities do not increase with a \emph{decreasing} discount percentage: $\mu_1 \geq \mu_2 \geq \dotsc \geq \mu_N$. Arm $i+1$ represents a \emph{smaller} discount than arm $i$: the experimental results use discount levels of $20\%,10\%,0\%$ for arms $1,2,3$ respectively. Our approach can be used also with a much more complex set of possible prior assumptions.  

Unlike in regular Bernoulli bandits, the rewards $r_i$ depend on arm $i$. That is, after pulling an arm (choosing a discount level) we received reward $r_i$ if the customer decides to purchase the product and $0$ otherwise. Since the arm discount decreases with the index $i$, the rewards satisfy: $ r_1 < r_2 < \ldots < r_N$. 
 
\begin{algorithm}  
	\KwIn{Current time step $t$, current state $s_t$, and value function $v_t: \states_t \rightarrow \Real$}
	\KwOut{Arm to pull at time step $t$} 
    $\mathcal{X} \gets \emptyset$ \\
    $(\alpha, \beta) \gets s_t$ \\
    \For{$k \in 1 \dotsc \operatorname{Sample Count}$}{
    	Sample $\theta_i \sim \betaD(\alpha, \beta) $ for each $a_i\in\actions$\\
        \If{$\theta_1 \geq \theta_2 \geq \ldots \geq \theta_n$}{
        	$\mathcal{X} \gets \mathcal{X} \cup \{(\hat{\mu}_1, \hat{\mu}_2, \dotsc, \hat{\mu}_N)\}$
        }
    }
    $\tilde{\mu}_i \gets \frac{1}{|\mathcal{X}|} \cdot \sum\limits_{\hat{\mu} \in \mathcal{X}} \hat{\mu}_i$ \tcp*{compute sample average} 
	\For{$i \in 1 \dotsc N$}{
        $q_t(s_t,a_i) \gets \tilde{\mu}_i \cdot \Big(r(s_t,a,(\alpha+1,\beta)) + v_{t+1}(\alpha+1,\beta)\Big)
        +  (1 - \tilde{\mu}_i) \cdot v_{t+1}(\alpha,\beta+1)$
	}
	\Return{$\arg\max_{a\in\actions} q_t(s_t,a)$ } \;
	\caption{Constrained single step lookahead algorithm}
    \label{alg:lookahead-constrained}
\end{algorithm}

Adapting \cref{alg:lookahead} to this constrained setting is relatively straightforward. We use the linearly separable value function computed by ELSV and only modify the lookahead to respect the constraints on $\mu$. In particular, we use \emph{rejection sampling} to appropriately adjust the transition probabilities when updating the values in the lookahead. As \cref{alg:lookahead-constrained} shows, we essentially compute the conditional probability distribution for each $\mu_i$ given the observations for that arm as well as the observations for other arms. This probability distribution must be estimated empirically as it does not have a closed form. The algorithm is only a heuristic in this setting and we have no regret bounds yet.

\plotfig{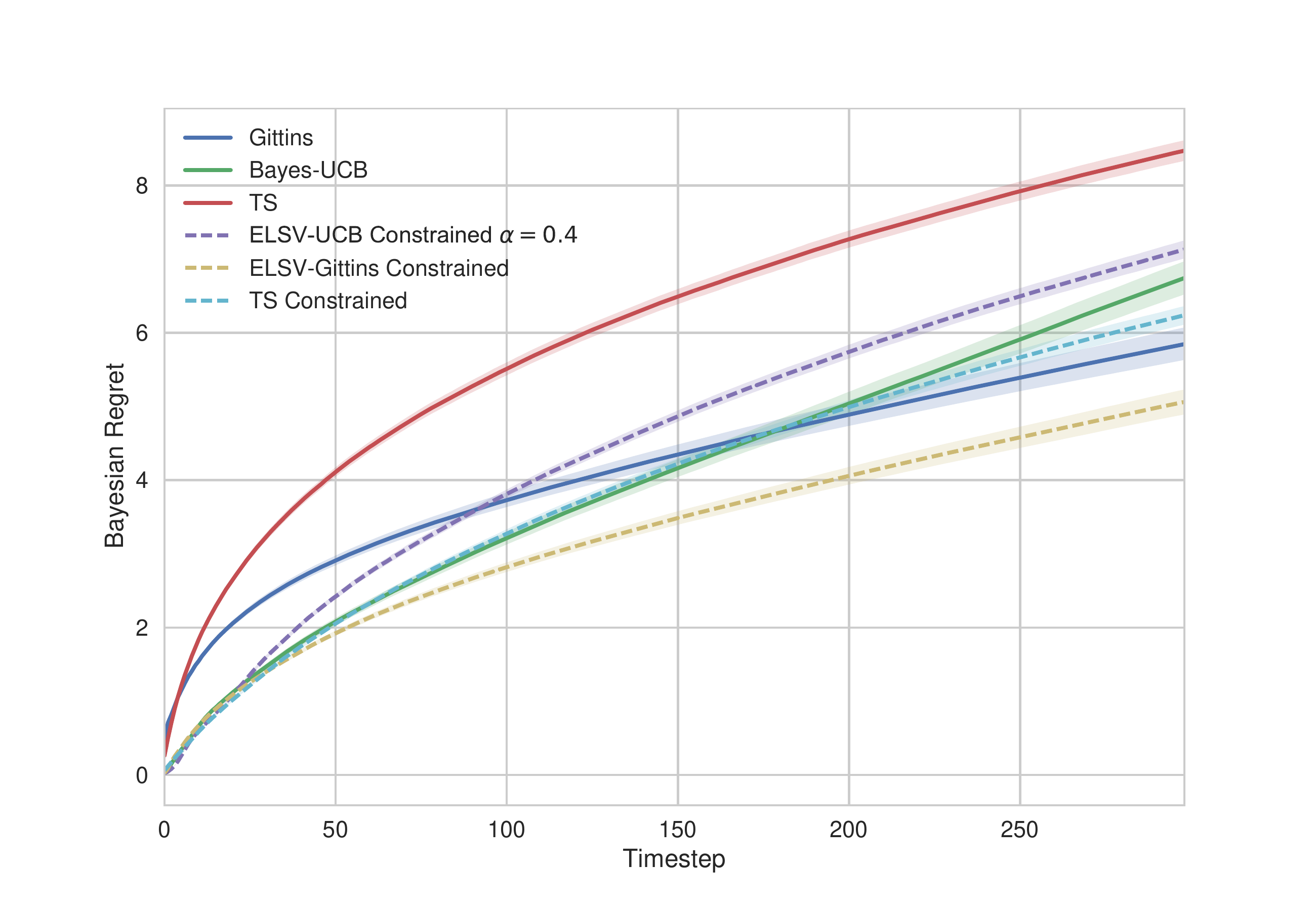}{Constrained Bernoulli bandit regret}{fig:constrained}

\cref{fig:constrained} shows the regret of ELSV (with 1-step lookahead) compared with several state-of-the-art algorithms on the constrained bandit problem with $3$ arms and averaged over $5300$ problem instances and with $95\%$ confidence intervals. ELSV-UCB and ELSV-Gittins use a value function computed from UCB and Gittins index respectively. We omit the regret of UCB from the plot because its regret was much higher than that of the other algorithms. TS stands for regular Thompson sampling that ignores the constraints on $\mu$, while ``TS Constrained'' samples from the constrained posterior distribution using rejection sampling similarly to \cref{alg:lookahead-constrained}. 

Our results show that ELSV-Gittins outperforms all other algorithms even in a problem with 3 arms. The magnitude of improvement in ELSV-Gittins over Gittins index grows as the number of arms in the problem increases since the constrained becomes more important and more informative. ELSV-UCB performs a bit worse, it but still represents a very significant improvement over plain UCB. 

\section{Conclusion}

We have proposed a new approach to bandit problems focused on good short-term performance in problems with structured prior information. The approach is based on a new kind of linearly separable value function that incorporates the value of exploration and can be used in concert with online planning methods. Our method, ELSV, performs close to optimal on basic Bernoulli bandits and can significantly outperform existing methods in problems with prior information. The results on simple bandit problems are promising and we hope to extend the approach also to contextual bandits. We also believe that ELSV is a good first step in developing more sophisticated value-directed methods in the future.

\bibliography{mareklib,reazullib}

\begin{thebibliography}{}

\bibitem[\protect\citeauthoryear{Adelman and Mersereau}{2008}]{Adelman2008}
Daniel Adelman and AJ~Mersereau.
\newblock {Relaxations of weakly coupled stochastic dynamic programs}.
\newblock {\em Operations Research}, 56(3):712--727, 2008.

\bibitem[\protect\citeauthoryear{Agrawal and Goyal}{2012}]{Agrawal2011}
Shipra Agrawal and N~Goyal.
\newblock {Analysis of Thompson sampling for the multi-armed bandit problem}.
\newblock In {\em Annual Conference on Learning Theory (COLT)}, pages
  39.1--39.26, 2012.

\bibitem[\protect\citeauthoryear{Audibert \bgroup \em et al.\egroup
  }{2009}]{Audibert2009}
Jean~Yves Audibert, R{\'{e}}mi Munos, and Csaba Szepesv{\'{a}}ri.
\newblock {Exploration-exploitation tradeoff using variance estimates in
  multi-armed bandits}.
\newblock {\em Theoretical Computer Science}, 410(19):1876--1902, 2009.

\bibitem[\protect\citeauthoryear{Auer \bgroup \em et al.\egroup
  }{2002}]{Auer2002}
P~Auer, N~Cesa-bianchi, and P~Fischer.
\newblock {Finite time analysis of the multiarmed bandit problem}.
\newblock {\em Machine Learning}, 47(2-3):235--256, 2002.

\bibitem[\protect\citeauthoryear{Bubeck and Cesa-Bianchi}{2012}]{Bubeck2012}
S{\'{e}}bastien Bubeck and Nicol{\`{o}} Cesa-Bianchi.
\newblock {Regret Analysis of Stochastic and Nonstochastic Multi-armed Bandit
  Problems}.
\newblock {\em Foundations and Trends in Machine Learning}, 5(1):1--122, apr
  2012.

\bibitem[\protect\citeauthoryear{Chakravorty and
  Mahajan}{2014}]{Chakravorty2014}
Jhelum Chakravorty and Aditya Mahajan.
\newblock {Multi-Armed Bandits, Gittins Index, and Its Calculation}.
\newblock In {\em Methods and Applications of Statistics in Clinical Trials},
  pages 416--435. 2014.

\bibitem[\protect\citeauthoryear{Filippi \bgroup \em et al.\egroup
  }{2010}]{Filippi2010}
Sarah Filippi, O~Cappe, A~Garivier, and C~Szepesv{\'{a}}ri.
\newblock {Parametric Bandits: The Generalized Linear Case.}
\newblock In {\em Neural Information Processing Systems (NIPS)}, 2010.

\bibitem[\protect\citeauthoryear{Gittins \bgroup \em et al.\egroup
  }{2011}]{Gittins2011}
John Gittins, Kevin Glazerbrook, and Richard Weber.
\newblock {\em {Multi-Armed Bandit Allocation Indices}}.
\newblock John Wiley {\&} Sons, 2nd edition, 2011.

\bibitem[\protect\citeauthoryear{Gittins}{1979}]{Gittins1979}
JC~Gittins.
\newblock {Bandit processes and dynamic allocation indices}.
\newblock {\em Journal of the Royal Statistical Society. Series B},
  41(2):148--177, 1979.

\bibitem[\protect\citeauthoryear{Gutin and Farias}{2016}]{Gutin2016}
Eli Gutin and Vivek~F. Farias.
\newblock {Optimistic Gittins Indices}.
\newblock In {\em Conference on Neural Information Processing Systems (NIPS)},
  2016.

\bibitem[\protect\citeauthoryear{Kaufmann \bgroup \em et al.\egroup
  }{2012a}]{Kaufmann2012a}
Emilie Kaufmann, Olivier Capp{\'{e}}, and Aur{\'{e}}lien Garivier.
\newblock {On Bayesian upper confidence bounds for bandit problems}.
\newblock {\em International Conference on Artificial Intelligence and
  Statistics}, pages 592--600, 2012.

\bibitem[\protect\citeauthoryear{Kaufmann \bgroup \em et al.\egroup
  }{2012b}]{Kaufmann2012}
Emilie Kaufmann, Nathaniel Korda, and R{\'{e}}mi Munos.
\newblock {Thompson Sampling: An Asymptotically Optimal Finite Time Analysis}.
\newblock {\em Algorithmic Learning Theory}, (1):15, 2012.

\bibitem[\protect\citeauthoryear{Kim and Lim}{2015}]{Kim2015}
Michael~Jong Kim and Andrew~E.B. Lim.
\newblock {Robust Multiarmed Bandit Problems}.
\newblock {\em Management Science}, 62(1):264--285, 2015.

\bibitem[\protect\citeauthoryear{Kuleshov and Precup}{2014}]{Kuleshov2014}
Volodymyr Kuleshov and Doina Precup.
\newblock {Algorithms for multi-armed bandit problems}.
\newblock Technical report, 2014.

\bibitem[\protect\citeauthoryear{Lattimore and
  Szepesv{\'{a}}ri}{2016}]{Lattimore2016}
Tor Lattimore and Csaba Szepesv{\'{a}}ri.
\newblock {The End of Optimism? An Asymptotic Analysis of Finite-Armed Linear
  Bandits}.
\newblock Technical report, 2016.

\bibitem[\protect\citeauthoryear{Lattimore}{2016}]{Lattimore2016a}
Tor Lattimore.
\newblock {Regret Analysis of the Finite-Horizon Gittins Index Strategy for
  Multi-Armed Bandits}.
\newblock In {\em Annual Conference on Learning Theory (COLT)}. arXiv, 2016.

\bibitem[\protect\citeauthoryear{Leike \bgroup \em et al.\egroup
  }{2016}]{Leike2016}
Jan Leike, Tor Lattimore, Laurent Orseau, and Marcus Hutter.
\newblock {Thompson Sampling is Asymptotically Optimal in General
  Environments}.
\newblock In {\em Uncertainty in Artificial Intelligence (UAI)}, 2016.

\bibitem[\protect\citeauthoryear{Ni{\~{n}}o-Mora}{2011}]{Nino-Mora2011}
Jos{\'{e}} Ni{\~{n}}o-Mora.
\newblock {Computing a classic index for finite-horizon bandits}.
\newblock {\em INFORMS Journal on Computing}, 23(2):254--267, 2011.

\bibitem[\protect\citeauthoryear{Petrik and Zilberstein}{2011}]{Petrik2011b}
Marek Petrik and Shlomo Zilberstein.
\newblock {Robust approximate bilinear programming for value function
  approximation}.
\newblock {\em Journal of Machine Learning Research}, 12(1):3027--3063, 2011.

\bibitem[\protect\citeauthoryear{Powell \bgroup \em et al.\egroup
  }{2004}]{Powell2004}
Warren Powell, Andrzej Ruszczynski, and Huseyin Topaloglu.
\newblock {Learning Algorithms for Separable Approximations of Discrete
  Stochastic Optimization Problems}.
\newblock {\em Mathematics of Operations Research}, 29(4):814--836, nov 2004.

\bibitem[\protect\citeauthoryear{Powell}{2008}]{Powell2008}
Warren~B Powell.
\newblock {Value Function Approximation using Multiple Aggregation for
  Multiattribute Resource Management}.
\newblock {\em Journal of Machine Learning Research}, 9:2079--2111, 2008.

\bibitem[\protect\citeauthoryear{Russo and Roy}{2014}]{Russo2014}
Daniel Russo and Benjamin~Van Roy.
\newblock {Learning to Optimize Via Information-Directed Sampling}.
\newblock pages 1--34, 2014.

\bibitem[\protect\citeauthoryear{Russo and {Van Roy}}{2014}]{Russo2014b}
Daniel Russo and Benjamin {Van Roy}.
\newblock {Learning to Optimize via Posterior Sampling}.
\newblock {\em Mathematics of Operations Research}, 39(4):1221--1243, 2014.

\bibitem[\protect\citeauthoryear{Rust}{1996}]{Rust1994}
John Rust.
\newblock {Numerical dynamic programming in economics}.
\newblock {\em Handbook of computational economics}, (November), 1996.

\bibitem[\protect\citeauthoryear{Sutton and Barto}{2016}]{Sutton2016}
R~S Sutton and A~G Barto.
\newblock {\em {Reinforcement Learning : An Introduction}}.
\newblock 2016.

\bibitem[\protect\citeauthoryear{Train}{2003}]{Train2003}
Kenneth~E. Train.
\newblock {\em {Discrete Choice Methods with Simulation}}.
\newblock 2003.

\bibitem[\protect\citeauthoryear{Whittle}{1988}]{Whittle1988}
P.~Whittle.
\newblock {Restless Bandits: Activity Allocation in a Changing World}.
\newblock {\em Journal of Applied Probability}, 25(1988):287, 1988.

\end{thebibliography}

\clearpage
\appendix
\onecolumn

\section{Proof of \cref{thm:value_property}} \label{sec:proof_value_property}

\begin{proof}
To prove the theorem, we need to show that \cref{alg:lookahead} will take the same action as the index policy for each state $s\in\states$ that:
\[ \arg\max_{a_i\in\actions} q_t(s,a_i) = \arg\max_{a_i\in\actions} \hat{z}_t(s^i,a_i)~. \]  
By subtracting a constant from the left-hand side, we get the following equivalent statement:
\[ \arg\max_{a_i\in\actions} \; \Bigl( q_t(s,a_i) - \hat{v}_{t+1}(s) \Bigr) = \arg\max_{a_i\in\actions} \hat{z}_t(s^i,a_i)~. \]  
Informally, the term $q_t(s,a_i) - \hat{v}_{t+1}(s)$ measures the advantage of pulling an arm $a_i$ in comparison with a virtual action of ``pulling no arm'' instead. Next we show that $q_t(s,a_i) - \hat{v}_{t+1}(s) = \hat{z}_t(s_i,a_i)$.

Fix an action $a_i$, let $s_\tau$ be some state, and let $S_{\tau+1}$ to be the random variable representing the state that follows $s_\tau^i$ after taking action $a_i$. Than:
\begin{align*}
q_t(s_\tau,a_i) - \hat{v}_{t+1}(s_\tau) &= \E{\hat{v}_{t+1}(S_{\tau+1}) + r(s_\tau,a_i,S_{\tau+1} )} - \hat{v}_{t+1}(s_\tau) = \\
&= \E{r(s_\tau,a_i,S_{\tau+1})} + \E{\hat{v}_{t+1}(S_{\tau+1})} - \hat{v}_{t+1}(s_\tau) = \\
&= r(s_\tau^i,a_i) + \E{\hat{v}_{t+1}(S_{\tau+1})} - \hat{v}_{t+1}(s_\tau) 
\end{align*}
Using the assumed linearly separable form of $\hat{v}_t(s_\tau) = \sum_{a\in\actions} \hat{\upsilon}^i_t(s_\tau^i)$, we get:
\begin{align*}
r(s_\tau^i,a_i) + \E{\hat{v}_{t+1}(S_{\tau+1})} - \hat{v}_{t+1}(s_\tau) &=  r(s_\tau^i,a_i) + \E{\sum_{a_j\in\actions} \hat{\upsilon}_{t+1}^{j}(S_{\tau+1}^{j})} - \sum_{a_j\in\actions} \hat{\upsilon}_{t+1}^{j}(s_\tau^{j}) = \\
&= r(s_\tau^i,a_i) + \sum_{a_j\in\actions} \Bigl( \E{ \hat{\upsilon}_{t+1}^{a'}(S_{\tau+1}^{j})} - \hat{\upsilon}_{t+1}^{j}(s_\tau^{j}) \Bigr)
\end{align*}
Notice next that $S_{\tau+1}^{j} = s_\tau^{j}$ whenever $a_j \neq a_i$, since the state of an arm does not change unless the arm is pulled. Thus we can further simplify the sum as follows:
\begin{gather*}
r(s_\tau^i,a) + \sum_{a_j\in\actions} \Bigl( \E{ \hat{\upsilon}_{t+1}^{j}(S_{\tau+1}^{j})} - \hat{\upsilon}_{t+1}^{j}(s_\tau^{j}) \Bigr) = \\
= r(s_\tau^i,a_i) + \sum_{a_j\in\actions \setminus \{a_i\}} \Bigl( \E{ \hat{\upsilon}_{t+1}^{j}(S_{\tau+1}^{j})} - \hat{\upsilon}_{t+1}^{j}(s_\tau^{j}) \Bigr) +  \E{ \hat{\upsilon}_{t+1}^{i}(S_{\tau+1}^{i})} - \hat{\upsilon}_{t+1}^{i}(s_\tau^{i}) = \\
= r(s_\tau^i,a_i) + \sum_{a_j\in\actions \setminus \{a_i\}} \Bigl( \E{ \hat{\upsilon}_{t+1}^{j}(s_\tau^{j})} - \hat{\upsilon}_{t+1}^{j}(s_\tau^{j}) \Bigr) +  \E{ \hat{\upsilon}_{t+1}^{i}(S_{\tau+1}^{i})} - \hat{\upsilon}_{t+1}^{i}(s_\tau^{i}) = \\
= r(s_\tau^i,a_i) + \E{ \hat{\upsilon}_{t+1}^{i}(S_{\tau+1}^{i})} - \hat{\upsilon}_{t+1}^{i}(s_\tau^{i})~.
\end{gather*}
Finally, substituting \eqref{eq:dynamic_update} into the equation above gives us:
\[q_t(s_\tau,a_i) - \hat{v}_{t+1}(s_\tau) = \hat{z}_t(s_\tau,a_i)~, \]
which proves the theorem.
\end{proof}

\section{Proof of \cref{lem:decomposition}}  \label{sec:regret_decomposition_proof}

\begin{proof}	
First, we will use the following decomposition of the regret for a given parameter $\mu$:
\begin{align*}
\regret(\pi, T, \mu) &= \sum_{t=1}^T \E{\muopt - \mu_{\pi(S_t)}} = \\
&\stackrel{\text{(a)}}{=} \sum_{t=1}^T \left(\E{\muopt - \bigl(q_t(S_t, \pi(S_t)) - v_{t+1}(S_t) \bigr)} + \E{\bigl(q_t(S_t, \pi(S_t)) - v_{t+1}(S_t) \bigr) - \mu_{\pi(S_t)}} \right) = \\
&\stackrel{\text{(b)}}{\le} \sum_{t=1}^T \left( \E{\mu_{\iopt} - \bigl(q_t(S_t, \aopt) - v_{t+1}(S_t) \bigr)} + \E{\bigl(q_t(S_t, \pi(S_t)) - v_{t+1}(S_t) \bigr) - \mu_{\pi(S_t)}} \right) = \\
&= \sum_{t=1}^T \E{\underbrace{\left(\muopt - q_t(S_t, a_{\iopt}) + v_{t+1}(S_t) \right)}_{-\varphi_t(\mu,S_t,\aopt)}} + \sum_{t=1}^T \E{\underbrace{\left(q_t(S_t, \pi(S_t)) - v_{t+1}(S_t) - \mu_{\pi(S_t)} \right)}_{\varphi_t(\mu,S_t,\pi(S_t))}} = \\
&= \sum_{t=1}^T \E[S_t]{\varphi_t(\mu,S_t,\pi(S_t))} - \sum_{t=1}^T \E[S_t]{\varphi_t(\mu,S_t,\aopt)}
\end{align*}
The inequality (a) follows by simply adding $0 = v_{t+1}(S_t) - v_{t+1}(S_t)$, and the inequality (b) follows from the optimality of $\pi(S_t)$ with respect to $v_t$:
\[ q_t(s_t,\pi(s_t)) \ge q_t(S_t, a),  \]
for any action $a\in\states$.
\end{proof}

\section{Proof of \cref{thm:regret}} \label{sec:thm_regret_proof}

The theorem follows directly from the following two lemmas.

\begin{lemma} \label{lem:proof_optimal}
For $\alpha = 0$ and $\varphi_t$ computed for $v^\ucb$ we have the following lower bound:
\[ \sum_{t=1}^T \E[S_t]{\varphi_t(\mu,S_t,\aopt)} \ge - \frac{1}{t^{2\alpha -1}} \ge O(1)~. \]
\end{lemma}

\begin{proof}
\begin{align*}
\E[S_t]{\varphi_t(\mu,S_t,\aopt)} &= \E[S_t]{q_t(S_t, \aopt) - v_{t+1}(S_t) - \muopt} = \E{z^\ucb(S_t,\aopt) - \muopt} =\\
&= \E{r(S_t, \aopt) + \sqrt{\frac{\alpha \log t}{\pulls_{\iopt}(S_t^{\iopt},\aopt)}} - \muopt} \ge \\
&\stackrel{\text{(a)}}{\ge} - \P{r(S_t,\aopt) -\muopt \le \sqrt{\frac{\alpha \log t}{\pulls_{\iopt}(S_t^{\iopt},\aopt)}}  }  \ge \\
&\stackrel{\text{(b)}}{\ge} - t \exp\left( - 2 \pulls_{\iopt}(S_t^{\iopt},\aopt) \frac{\alpha \log t}{\pulls_{\iopt}(S_t^{\iopt},\aopt)} \right) =\\
&= - t \exp(-2\alpha \log t) = -\frac{1}{t^{2\alpha -1}}
\end{align*}	
where (a) follows by rewriting expectation and the fact that the rewards are bounded between 0 and 1, and (b) follows from Hoeffding's inequality and the union bound over possible values of $T_{\iopt}$. Summing the value above for $t = 1\ldots T$ proves the lemma.
\end{proof}

\begin{lemma} \label{lem:proof_anyarm}
Fix an action $a_i \in \actions$, and for a policy $\pi$ greedy with respect to $v^\ucb$ then:
\[ \sum_{t=1}^T \E[S_t]{\varphi_t(\mu,S_t,\pi(S_t))} \le O(\sqrt{T \log T}) \]
\end{lemma}
\begin{proof}
Consider the bound for a single action:
\begin{align*}
\E[S_t]{\varphi_t(\mu,S_t,a_i)} &= \E[S_t]{q_t(S_t, a_i) - v_{t+1}(S_t) - \mu_i} = \E{z^\ucb(S_t,a_i) - \muopt} \\
&= \E{r(S_t, a_i) + \sqrt{\frac{\alpha \log t}{\pulls_{i}(S_t^i,a_i)}} - \mu_i} = \\
&= \E{r(S_t, a_i) - \sqrt{\frac{\alpha \log t}{\pulls_{i}(S_t^i,a_i)}} + 2 \sqrt{\frac{\alpha \log t}{\pulls_{i}(S_t^i,a_i)}} - \mu_i} = \\
&= \E{r(S_t, a_i) - \sqrt{\frac{\alpha \log t}{\pulls_{i}(S_t^i,a_i)}} - \mu_i} + \E{2 \sqrt{\frac{\alpha \log t}{\pulls_{i}(S_t^i,a_i)}}} =  \\
&\stackrel{\text{(a)}}{\le} \P{r(S_t, a_i) - \mu_i \le \sqrt{\frac{\alpha \log t}{\pulls_{i}(S_t^i,a_i)}}} + \E{2 \sqrt{\frac{\alpha \log t}{\pulls_{i}(S_t^i,a_i)}}} \\
&\stackrel{\text{(c)}}{\le} \frac{1}{t^{2\alpha -1}} + \E{2 \sqrt{\frac{\alpha \log t}{\pulls_{i}(S_t^i,a_i)}}} \\
\end{align*}
The inequalities above follow readily using algebra; (a) and (b) follows by an identical argument to the proof of \cref{lem:proof_optimal}.

Considering just the state in which arm $a_i$ is pulled, it remains to upper bound the following expression:
\begin{gather*} 
\sum_{t=1}^T \E{2\, \one_{a_i = \pi(S_t)} \sqrt{\frac{\alpha \log t}{\pulls_{i}(S_t^i,a_i)}}} \le \sum_{t=1}^T \E{2\, \one_{a_i = \pi(S_t)} \sqrt{\frac{\alpha \log T}{\pulls_{i}(S_t^i,a_i)}}} \le \\
\stackrel{\text{(d)}}{\le} \sum_{t=1}^T \E{2\, \sqrt{\frac{\alpha \log T}{t}}} \stackrel{\text{(e)}}{\le} O(\sqrt{T \log T})~.
\end{gather*}
The inequality (d) follows by upper bounding the error by assuming that the arm was pulled in every step, and (e) follows by upper bounding the sum by an integral. See, for example, the proof of Proposition 2 in \citeasnoun{Russo2014b}.
\end{proof}


\end{document}